%% file: root.tex
\documentclass[letterpaper, 10 pt, conference]{ieeeconf}
\usepackage{amsthm}
\usepackage[english]{babel}

\IEEEoverridecommandlockouts
\usepackage{cite}
\usepackage{algorithmic}  % algorithmicx
\usepackage{algorithm}
\usepackage{amsmath,amssymb,amsfonts}
\usepackage{graphicx}
\usepackage{textcomp}
\usepackage{xcolor}
\usepackage{colortbl}
\usepackage{multirow}
\usepackage{booktabs}
\usepackage[colorlinks,linkcolor=blue]{hyperref}
\usepackage{siunitx}
\usepackage{bm}
\usepackage{graphicx}

\usepackage{threeparttable}

\usepackage{booktabs}

\usepackage{listings}
\usepackage{xcolor}

\usepackage{graphicx} % 用于插入图片
\usepackage{subcaption} % 用于子图
\usepackage{lipsum} % 用于生成示例文本

\newtheorem{theorem}{Theorem}

\def\BibTeX{{\rm B\kern-.05em{\sc i\kern-.025em b}\kern-.08em
    T\kern-.1667em\lower.7ex\hbox{E}\kern-.125emX}}
\begin{document}

\include{defs}

% conf title
\title{\LARGE \bf Transferable Latent-to-Latent Locomotion Policy for Efficient and Versatile Motion Control of Diverse Legged Robots}

\author{Ziang Zheng$^{\#1}$, Guojian Zhan$^{\#1}$, Bin Shuai$^{1}$, Shengtao Qin$^{1}$, Jiangtao Li$^{2}$, Tao Zhang$^{2}$, Shengbo Eben Li$^{*1}$
% <-this % stops a space
\thanks{This study is supported by National Key R\&D Program of China with 2022YFB2502901 and NSF China under 52221005. $^{\#}$The first two authors contributed equally. $^{*}$All correspondence should be sent to Shengbo Eben Li. {\tt\small Email: lishbo@mail.tsinghua.edu.cn}}% <-this % stops a space
\thanks{$^{1}$State Key Laboratory of Intelligent Green Vehicle and Mobility, School of Vehicle and Mobility, Tsinghua University, Beijing, China.
$^{2}$SunRising AI Ltd, Beijing, China..
}%
\thanks{}% <-this % stops a space
}

\maketitle
\begin{abstract}
Reinforcement learning (RL) has demonstrated remarkable capability in acquiring robot skills, but learning each new skill still requires substantial data collection for training. The pretrain-and-finetune paradigm offers a promising approach for efficiently adapting to new robot entities and tasks.
Inspired by the idea that acquired knowledge can accelerate learning new tasks with the same robot and help a new robot master a trained task, we propose a latent training framework where a transferable latent-to-latent locomotion policy is pretrained alongside diverse task-specific observation encoders and action decoders.
This policy in latent space processes encoded latent observations to generate latent actions to be decoded, with the potential to learn general abstract motion skills. To retain essential information for decision-making and control, we introduce a diffusion recovery module that minimizes information reconstruction loss during pretrain stage.
During fine-tune stage, the pretrained latent-to-latent locomotion policy remains fixed, while only the lightweight task-specific encoder and decoder are optimized for efficient adaptation.
Our method allows a robot to leverage its own prior experience across different tasks as well as the experience of other morphologically diverse robots to accelerate adaptation.
We validate our approach through extensive simulations and real-world experiments, demonstrating that the pretrained latent-to-latent locomotion policy effectively generalizes to new robot entities and tasks with improved efficiency.
\end{abstract}

\input{S1-intro}
\input{S2-plm}
\input{S3-mth}
\input{S4-exp}

\section{Conclusion}
\label{sec:conclusion}

In this work, we proposed the Latent-to-Latent Locomotion Policy (L3P) framework, which introduces a task-independent latent space optimization process to enhance the efficiency of locomotion policy transfer. 
By utilizing a diffusion-based recovery module and joint optimizing policy performance and latent reconstruction loss, our approach ensures that the latent-to-latent policy encapsulates well-established skills and is ready for transfer without information loss. 
Through a series of experiments, we demonstrated the effectiveness of L3P in both single-to-diverse entity transfer and simple-to-complex terrain adaptation. 
The results highlight the potential of latent-to-latent policies in enabling versatile and efficient robotic locomotion. Future work will focus on refining the recovery module and expanding our framework to tackle more complex multi-entity and multi-task scenarios.
In summary, the L3P framework offers a powerful and ease-to-implement solution to the challenge of general motion control of diverse legged robots, paving the way for more adaptable and robust robotic systems capable of handling diverse environments and tasks.

\bibliographystyle{IEEEtran}
\bibliography{reference.bib}

% \appendix
% \input{Appendix}

\end{document}

%% file: defs.tex
%% abbreviations
\newcommand{\x}{\mathbf{x}}
\newcommand{\z}{\mathbf{z}}
\newcommand{\y}{\mathbf{y}}
\newcommand{\w}{\mathbf{w}}
\newcommand{\data}{\mathcal{D}}

%% specifics for the paper
\newcommand{\reward}{r}
\newcommand{\policy}{\pi}
\newcommand{\mdp}{\mathcal{M}}
\newcommand{\states}{\mathcal{S}}
\newcommand{\actions}{\mathcal{A}}
\newcommand{\observations}{\mathcal{O}}
\newcommand{\transitions}{T}
\newcommand{\freq}{d}
\newcommand{\obsfunc}{E}
\newcommand{\initial}{\mathcal{I}}
\newcommand{\horizon}{H}
\newcommand{\rewardevent}{R}
\newcommand{\probr}{p_\rewardevent}
\newcommand{\metareward}{\bar{\reward}}
\newcommand{\discount}{\gamma}
\newcommand{\behavior}{{\pi_\beta}}
\newcommand{\bellman}{\mathcal{B}}
\newcommand{\qparams}{\phi}
\newcommand{\qparamset}{\Phi}
\newcommand{\qset}{\mathcal{Q}}
\newcommand{\batch}{B}
\newcommand{\qfeat}{\mathbf{f}}
\newcommand{\Qfeat}{\mathbf{F}}
\newcommand{\traj}{\tau}
\newcommand{\return}{\mathcal{R}}

%% file: S1-intro.tex
\section{Introduction}

Legged robots are gaining increasing attention due to their ability to navigate complex terrains and perform tasks in environments that are challenging for wheeled or tracked robots. Their biological design allows them to traverse uneven surfaces\cite{bellicoso2018advances}, climb obstacles\cite{cheng2024extreme}, and maintain stability, making them ideal for applications in search and rescue, exploration\cite{lindqvist2022multimodality}, and industrial operations~\cite{arm2024pedipulate, zhang2024learning}. 
While optimization-based controllers such as model predictive control can handle structured environments, they still require meticulous platform-specific tuning and often fail to generalize across diverse morphologies or unstructured terrains. 

In contrast, learning-based methods have emerged as a promising alternative\cite{wellhausen2021rough}. 
These techniques employ a neural network as a controller, mapping the robot’s observations to control inputs. 

The most representative paradigm is imitation learning (IL) which trains the neural network to align the observation-action pairs collected from expects \cite{peng2018deepmimic, peng2018sfv, li2023learning}. Though there have been several remarkable demos achieved by IL, considerable high-quality data is still needed for continual performance improvement. RL does not rely on labeled data but learn optimal policies through trial and error by interacting with their environments~\cite{li2023rlbook, duan2021distributional}. This paradigm has been successfully supported across several domains including Chinese Go~\cite{schrittwieser2020mastering}, video games~\cite{vinyals2019grandmaster}. 
In legged robot locomotion field, RL has also demonstrated great potential in developing agile and robust locomotion skills directly from interaction with the environment \cite{peng2020learning} in both simulations \cite{liu2018learning,lee2019scalable,longhybrid, smith2022legged} and real-world environments \cite{hanna2017grounded}.
However, these learning-based locomotion controllers generally face two critical barriers: 
(1) the resulting policies are tightly coupled to a robot's specific kinematic design, hindering cross-platform transfer\cite{li2024experience}, and 
(2) the robot specific status gap like \textit{sim2real}, different motor conditions and robot loads persists, causing the dynamic mismatches between different entities \cite{tan2018sim, margolis2023walk}.

\begin{figure}[t]
\centering
\includegraphics[width=0.9\linewidth]{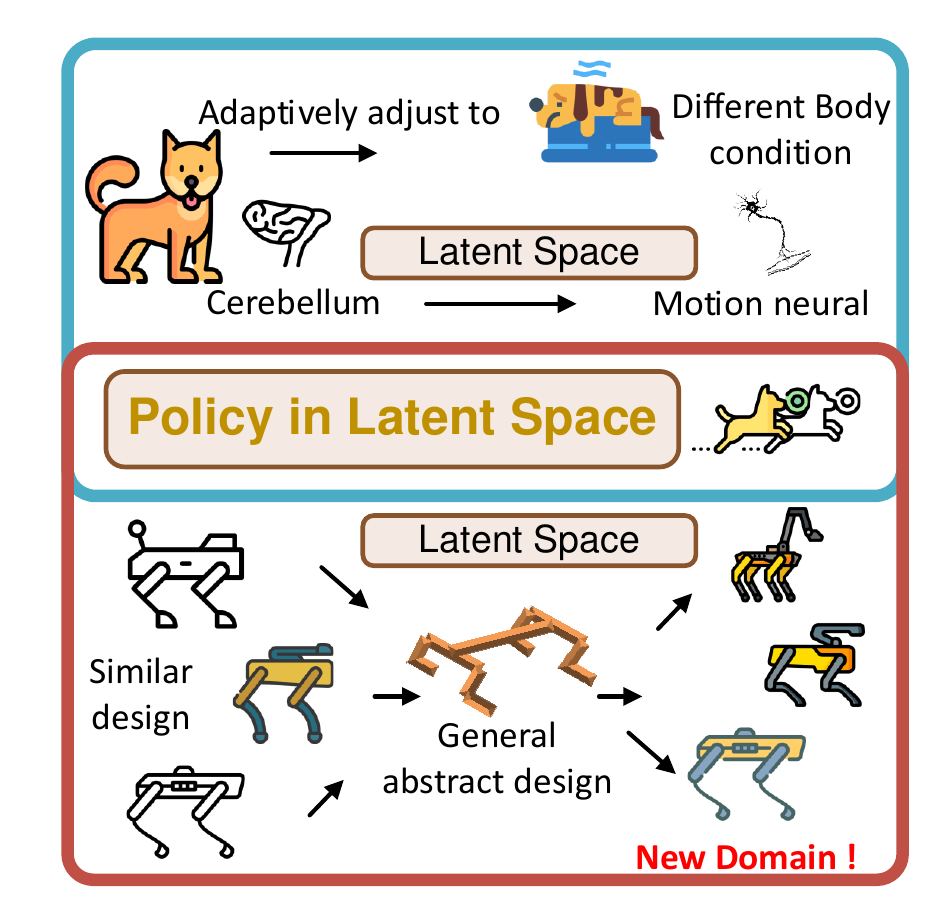}
\caption{
Biological inspiration for latent-to-latent locomotion policy.
}
\label{fig:bio_inspiration}
\end{figure}

Generally speaking, in many practical scenarios, the requisite training conditions are either inaccessible or insufficient \cite{chaffre2020sim}. 
For example, it is often impractical to train a robot in every possible complex or unstructured environment. 
This limitation underscores the need for approaches that can leverage pre-trained models or shared knowledge to improve the adaptation efficiency.
We can draw inspiration from humans and animals that can maintain an unconscious, latent model of their body, which is continuously adjusted in the cerebellum and translated into motor actions through neural signals as shown in Fig. \ref{fig:bio_inspiration}. 
Thereby a natural idea rises:
\begin{quote}
    \emph{Can we abstract both observation and action into the latent space to enable more efficient policy transfer?}
\end{quote}
Hence we propose a latent training framework, \textbf{L3P} (Latent-to-Latent Locomotion Policy) to enable skill transfer and platform adaptation.
Unlike traditional approaches that directly map robot-specific observation to actions, our framework maps raw observations through robot-specific encoder from diverse robots—regardless of their physical differences—into a shared latent space.
Then the policy outputs latent actions that serve as prototypes of robot behavior, which can be decoded as real control command by task-specific action decoder. 
Therefore, the trained latent-to-latent has the potential to contain general abstract motion skills and can be efficiently transferred.
To ensure that both the latent observations and latent actions capturing sufficient information for effective decision-making, we integrate a \emph{diffusion recovery module} that leverages the powerful multi-modal fitting capability of diffusion models. 
This module guarantees that the latent representations are capable of faithfully recovering the original information, thus enabling seamless encoder-decoder switching and cross-platform policy transfer.

Our contributions are summarized as follows:
\begin{itemize}
    \item We propose a Latent-to-Latent Locomotion policy (L3P) framework that decomposes control into observation encoder, latent policy, and action decoder. This design enables the observation encoder to adapt across diverse terrains while the action decoder can be tailored to different robotic individuals, thereby learning core universal skills that are readily reusable. Our design allows a robot to leverage its own prior experience across different tasks as well as the experience of other morphologically diverse robots to accelerate adaptation.
    \item We introduce a diffusion recovery module to ensure accurate reconstruction of latent states, maintaining consistency between the encoder-decoder pairs and ensuring that the latent representations capture sufficient critical information for decision-making. By utilizing a diffusion-based recovery module and joint optimizing policy performance and latent reconstruction loss, our approach ensures that the latent-to-latent policy encapsulates well-established skills and is ready for transfer without information loss. 
    \item Extensive experiments in simulation and on real-world platforms show that our approach achieves up to 2$\times$ higher zero-shot transfer success rates and adapts 3$\times$ faster to new platforms compared to baseline RL policies.
\end{itemize}

% The remainder of this paper is organized as follows: Section~\ref{sec:preliminary} outlines technical preliminaries. Section~\ref{sec:method} details L3P's design, followed by experiments in Section~\ref{sec:experiments}. Section~\ref{sec:conclusion} discusses limitations and future directions.

\begin{figure*}
    \centering
    \includegraphics[width=1.0\linewidth]{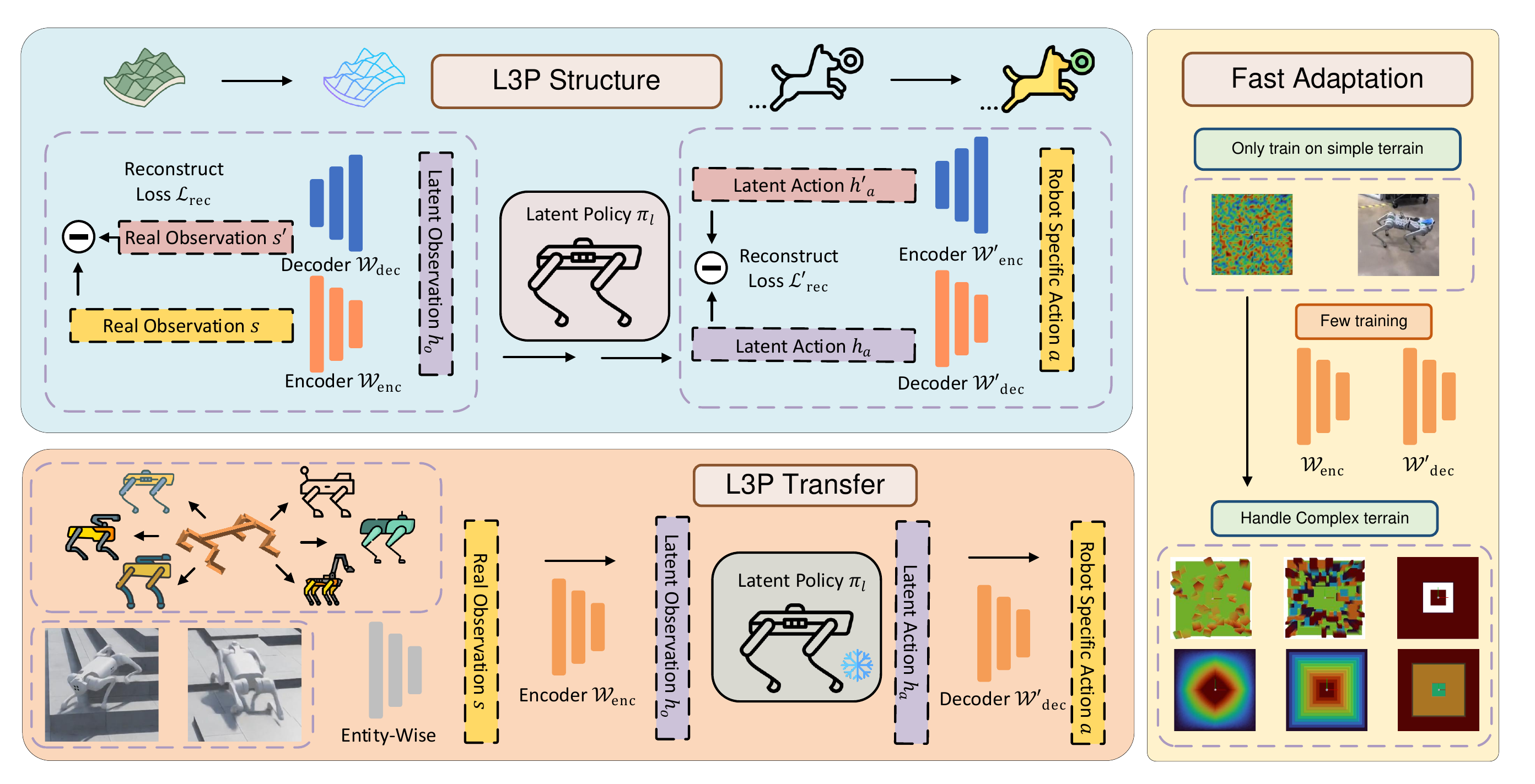}
    \caption{
    Overview of the L3P framework, which enables transferable locomotion control across diverse legged robots. 
    The framework consists of three key modules: (1) an observation latent encoder that maps raw sensory inputs into a shared latent space, (2) a latent policy backbone that learns a generalizable control strategy, and (3) a latent action decoder that translates latent actions into robot-specific motor commands. 
    Recovery modules for both observation and action ensure latent space consistency. The framework is trained in three stages: (i) latent space definition and alignment, (ii) transfer from a single entity to diverse robots, and (iii) zero-shot generalization from simple to complex terrains.
    }
    \label{fig:framework}
\end{figure*}

%% file: S2-plm.tex
\section{Preliminaries}
\label{sec:preliminary}
\subsection{Reinforcement Learning}
We consider the standard Markov Decision Process setting, where an agent interacts with an environment defined by a state space $\mathcal{S}$, an action space $\mathcal{A}$, and a reward function $r(s,a)$, with the goal of maximizing the expected cumulative reward:
\begin{equation}
J(\pi) = \mathbb{E}_{\pi}\left[\sum_{t=0}^{T} \gamma^t\, r(s_t, a_t)\right],
\end{equation}
where $\gamma \in [0,1)$ is the discount factor.

\subsection{Diffusion Models for Recovery}
Diffusion models have proven effective at capturing complex, multimodal distributions. Given a latent action $z^0$, the forward diffusion process incrementally adds Gaussian noise:
\begin{equation}
\label{eq:df_forward}
q(z^t\,|\,z^{t-1}) = \mathcal{N}\left(z^t; \sqrt{1-\beta_t}\,z^{t-1},\,\beta_t\mathbf{I}\right),
\end{equation}
where $\beta_t$ denotes the noise variance schedule and $t=1,\dots,T$. The reverse process, parameterized as
\begin{equation}
\label{eq:df_backward}
p_\theta(z^{t-1}\,|\,z^t) = \mathcal{N}\Big(z^{t-1}; \mu_\theta(z^t,t),\, \Sigma_\theta(z^t,t)\Big),
\end{equation}
is employed to reconstruct the original information from its noisy counterpart. In our framework, this reverse diffusion process serves as the \emph{diffusion recovery module} for more concisely reconstruction from the latent signal to the real-world signal thanks to the powerful multimodal fitting capability of diffusion model.

%% file: S3-mth.tex
\section{Method}
\label{sec:method}

In this section, we present our Latent-to-Latent Locomotion Policy (L3P) framework, which enables effective locomotion control across diverse robotic platforms. 

Our approach is built upon two core components: 
(1) the L3P architecture, which establishes a shared latent space and ensures its consistency through a recovery module, and  
(2) a transfer training strategy that facilitates adaptation from a single prototype to multiple entities and enables zero-shot generalization from simple to complex terrains. 
Each component corresponds to a distinct stage in the training process, ensuring a structured and scalable learning pipeline.

\subsection{Latent-to-Latent Policy Structure}

Fig.~\ref{fig:framework} illustrates our Latent-to-Latent Locomotion Policy (L3P) framework, which consists of three primary modules: 
an observation latent encoder, a latent-to-latent policy backbone, and a latent action decoder. 
Each encoder and decoder pair is complemented by a recovery module, ensuring that the latent representations are both information-rich and consistently aligned. 
In our architecture, the policy operates entirely in the latent space by mapping latent observations to latent actions, which are then decoded into robot-specific commands.

\subsubsection{Latent Space Definition and Alignment}

\textbf{Our framework searches for an optimal latent representation during policy optimization. }
In standard end-to-end training, intermediate layer outputs serve as latent representations\cite{li2022human}, yet their information content may vary due to task specifics, object characteristics, and training randomness. 
To ensure that the latent space consistently encapsulates all task-relevant features, we introduce a recovery module with an associated reconstruction loss. 
This additional constraint forces the encoder to discover a latent representation that is both sufficiently rich and robust—an optimal latent that minimizes the policy loss while retaining the necessary information.

We formalize this search as a joint optimization problem. Let $ s \in \mathcal{S} $ be an observation, and let the observation encoder $ W_{enc} $ map $ s $ into a latent representation $ h = W_{enc}(s) $. 
The policy then operates in the latent space, and we denote its objective as $ \mathcal{L}_{policy}(h) $. 
Concurrently, a recovery module $ W_{dec} $ attempts to reconstruct the latent representation, producing $s' = W_{dec}(z)$, with a reconstruction loss defined by
\begin{equation}
\mathcal{L}^s_{rec} = \|s - s'\|^2.
\end{equation}
Similarly for the action encoder-decoder pair, there also exists a reconstruction loss as
\begin{equation}
\mathcal{L}^a_{rec} = \|h_a - h'_a\|^2.
\end{equation}
The overall objective function becomes
\begin{equation}
\begin{aligned}
    \mathcal{L} &= \mathcal{L}_{policy}(h) + \lambda\, (\mathcal{L}^s_{rec} + \mathcal{L}^a_{rec}), \\
    &= \mathcal{L}_{policy}(h) + \lambda\, \mathcal{L}_{rec},
\end{aligned}
\end{equation}
where $\lambda$ is a hyper-parameter that balances policy performance and latent alignment.

\begin{theorem}[Optimal Latent Space Alignment]
\label{thm:latent_opt_align}
Assume that $ \mathcal{L}_{policy}(h) $ is a differentiable function that measures policy performance in the latent space, and $\mathcal{L}_{rec}$ enforces the recover-ability of $ z $. Then, under joint minimization of $\mathcal{L} = \mathcal{L}_{policy}(h) + \lambda\, \mathcal{L}_{rec}$, the training process implicitly searches for an optimal latent representation $ h^* $ that is a sufficient statistic for the observation $ s $ and minimizes the policy loss. Formally, if $ h^* $ is such that
\begin{equation}
h^* = \arg\min_{h} \left\{ \mathcal{L}_{policy}(h) + \lambda\, \mathcal{L}_{rec} \right\},
\end{equation}
then $ h^* $ is both discriminative for decision-making and stable in the sense that its reconstruction error is minimized.
\end{theorem}

\begin{proof}
Consider the joint loss function
\begin{equation}
\mathcal{L}(h) = \mathcal{L}_{policy}(h) + \lambda\, \mathcal{L}_{rec}.
\end{equation} 
Let $ h^* $ denote the latent representation that minimizes $\mathcal{L}(h)$. As training converges, the reconstruction loss $\|s - W_{dec}(h^*)\|^2$ approaches zero, implying that $ W_{dec}(h^*) \approx s $. 
Similarly we can also conclude $ h_a \approx h’_a $.
This convergence guarantees that $ h^* $ retains all the essential information from $ s $, making it a sufficient statistic. Concurrently, since $ h^* $ also minimizes $\mathcal{L}_{policy}(h)$, it is optimized for decision-making. Thus, the joint minimization ensures that $ h^* $ is both stable (low reconstruction error) and effective (low policy loss). This completes the proof.
\end{proof}

By enforcing this latent space alignment through the recovery module, our framework ensures that the latent representations are consistent and information-rich.
Both the action and observation perform this mechanism where the observation is reconstructed from latent and the latent action is reconstructed from action.

Detail to the training process, at the start of training, we select a simple-to-train entity and corresponding tasks requiring basic abilities. 
Hence the learned policy could perform enough basic actions and corresponding latent actions.
Once converged, we consider the latent action receives the latent observation and output latent action in target space.

\subsubsection{Diffusion Recovery Module}  
To mitigate the ambiguity where a single latent vector may map to multiple real-world states—especially in the observation space—we incorporate a diffusion-based recovery module.  
This module employs a U-Net \cite{ho2020denoising}, using the latent variable $ h $ as a conditioning input to progressively denoise from uniform noise.  
By harnessing the inherent multi-modal nature of diffusion models, it enhances the expressiveness and consistency of latent representations.

\subsection{Transfer Training}

Our L3P framework is designed to facilitate transfer learning in two key scenarios: (1) transferring from a single entity to diverse entities, and (2) transferring from simple to complex terrains.

\subsubsection{Single Entity to Diverse Entities}
Initially, the entire framework is trained on a single robotic platform across a wide range of environments to establish a robust latent space and to learn a universal latent policy $ \pi_{l} $. 
Once this latent space is defined, the latent policy is frozen. 
When transferring to a new robotic platform, only the observation encoder $ E_{obs} $ and the action decoder $ D_{act} $ are fine-tuned to align with the pre-established latent space. 
This fine-tuning process ensures that the new platform’s sensorimotor characteristics are accurately mapped into the latent space, thereby enabling rapid adaptation without re-learning the core locomotion skills.

\subsubsection{Simple Terrain to Difficult Terrain}
In scenarios where training on complex terrains is impractical, we perform fine-tuning on a simpler terrain (e.g., flat ground) for a given robotic platform.
During this phase, the observation encoder and action decoder are adjusted so that they conform to the latent space defined by the fixed latent policy. 
Although the fine-tuning occurs on a less challenging environment, the robust latent policy—enhanced by the diffusion recovery module—can generate latent actions that capture the necessary modalities. 
These aligned latent actions generalize well when the robot is subsequently deployed in more complex terrains, ensuring high performance even under challenging conditions.

% The complete training process for our L3P framework, incorporating the diffusion recovery module, is summarized in Algorithm~\ref{alg:l3p}.

% \begin{algorithm}[t]
% \caption{Training Procedure for L3P with Diffusion Recovery}
% \label{alg:l3p}
% \begin{algorithmic}[1]
% \STATE Initialize parameters for $E_{obs}, \pi_{l}, D_{act}, R_{obs}, \theta$.
% \WHILE{not converged}
%     \STATE Sample $ s $, compute $ z_{obs} = E_{obs}(s) $ and $ \hat{z}_{obs} = R_{obs}(z_{obs}) $.
%     \STATE Generate $ z_{act} = \pi_{l}(z_{obs}) $, decode action $ a = D_{act}(z_{act}) $.
%     \STATE Execute $ a $, obtain reward $ r $, next state $ s' $.
%     \STATE Apply forward and reverse diffusion \cite{ho2020denoising} to reconstruct $ \hat{z}_{act} $.
%     \STATE Compute $\mathcal{L}_{rec}^{obs}, \mathcal{L}_{rec}^{act}, \mathcal{L}_{policy}$, update parameters.
% \ENDWHILE
% \STATE Fine-tune $E_{obs}, D_{act}$ for transfer learning.
% \end{algorithmic}
% \end{algorithm}

%% file: S4-exp.tex
\begin{figure}
    \centering
    \includegraphics[width=1\linewidth]{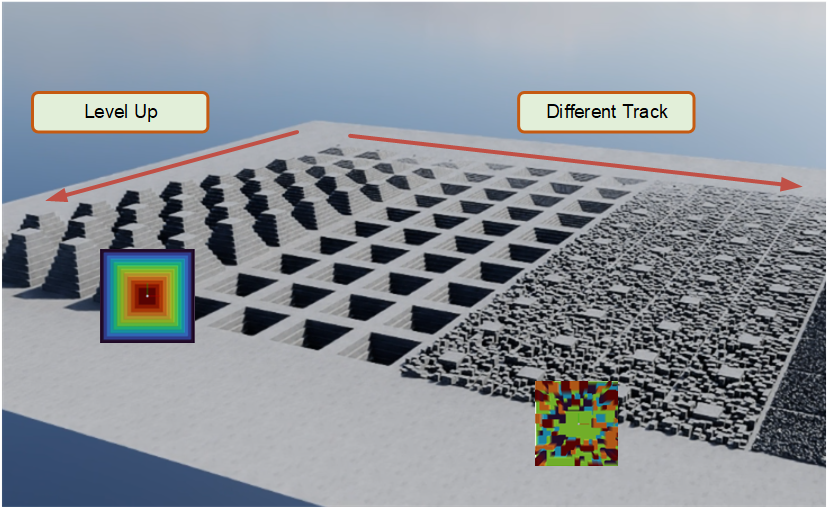}
    \caption{
    Overview of the simulation playground used for training, implemented in Isaac Sim. 
    The environment is structured such that different tracks are arranged horizontally, while terrain difficulty increases progressively in the vertical direction. 
    As the level rises, the tasks become increasingly challenging.
    }
    \label{fig:playground}
\end{figure}

\input{table/terrain_tab}

\input{table/main_pics}

\input{table/mass_cmp}

\begin{figure}[t]
\centering
\includegraphics[width=0.5\textwidth,trim={2cm 0.1cm 1.0cm 1.5cm}, clip]{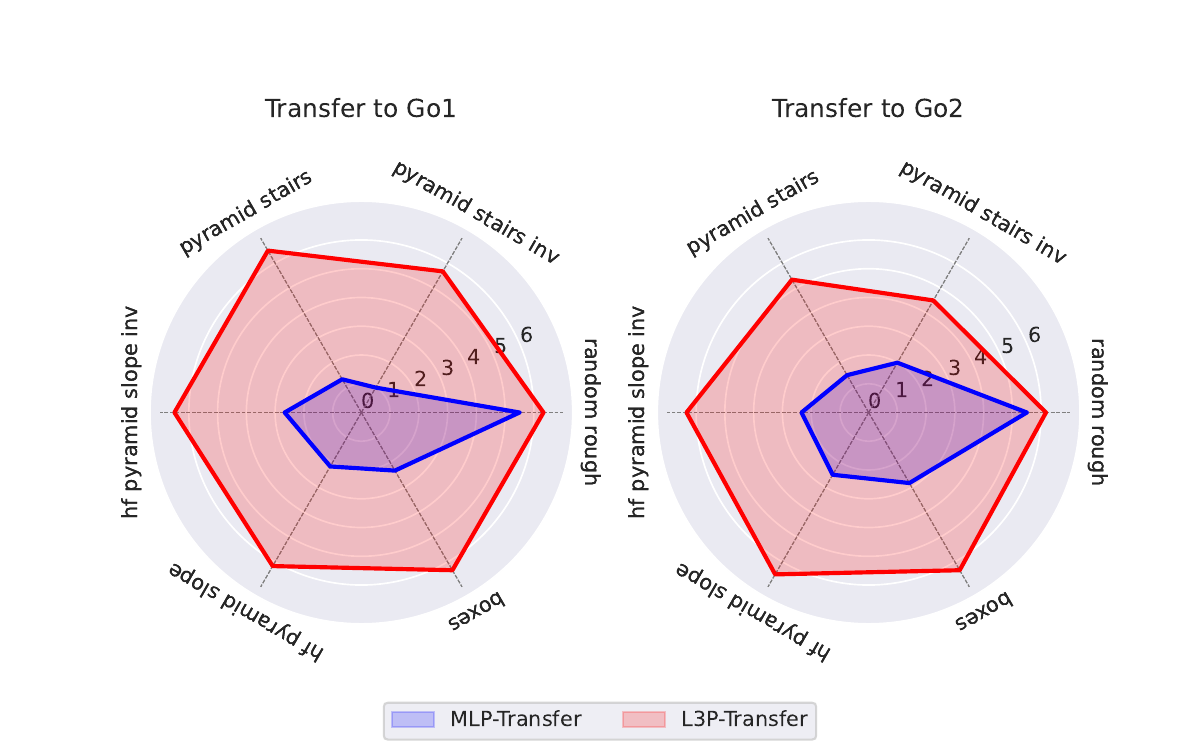}
\caption{
Performance verification by comparing L3P-Transfer with MLP-Transfer.
Each part represents an ability of passing corresponding terrain. The value represents the final level the policy could reach in the end through insufficient training on a simple flat terrain.
% And the policy is executing as long as possible to reflect the optimal result of each method. 
}
\label{fig:transfer_radar}
\end{figure}

\begin{figure}
\centering
\begin{subfigure}[b]{0.49\textwidth}
    \includegraphics[width=\textwidth]{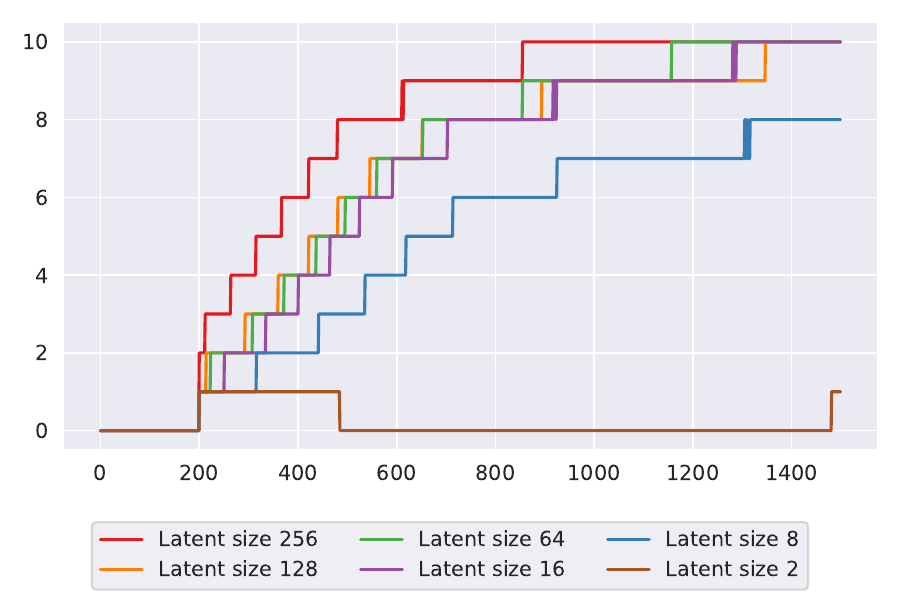}
     \caption{Latent observation }
    \label{obs_ablation}
\end{subfigure} \\
\begin{subfigure}[b]{0.49\textwidth}
    \includegraphics[width=\textwidth]{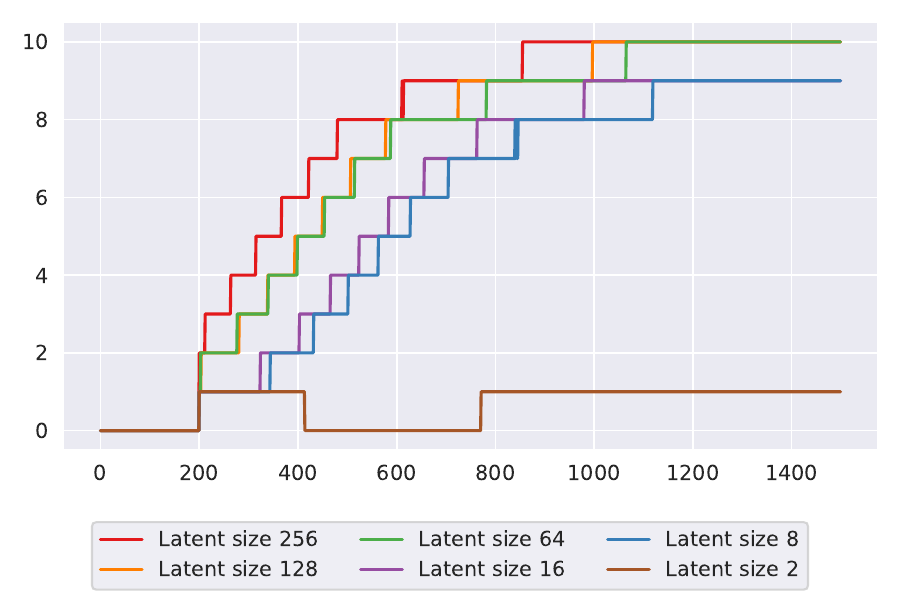}
     \caption{Latent action}
    \label{action_ablation}
\end{subfigure}
    \caption{Ablation study on the dimension of latent variables for both observation and action. }
    \label{fig:ablation}
\end{figure}

% \input{table/low_dim_curri}

% \begin{figure}
%     \centering
%     \includegraphics[width=1\linewidth]{figure/dist_act.png}
%     \caption{Distribution of 2 dimension latent action. The cluster group is calculation by the real actions. and we select the real actions that most close to cluster center and mark the corresponding latent action with red cross. }
%     \label{fig:latent_action_distribution}
% \end{figure}

\section{Experiment}
\label{sec:experiments}

We conduct extensive experiments in both simulation and real-world settings to evaluate the performance of our proposed L3P framework. Our experiments are designed to demonstrate three key aspects: 
(i) the same functionality of our L3P structure as generally adopted MLP in learning locomotion skills; 
(ii) the transferability of the pretrained latent-to-latent policy across different robotic platforms and terrain complexities. 
(iii) the effectiveness of our L3P structure in fast adaptation from simulation to real world robots.

\subsection{Simulation Experiment Setup}

In this subsection, we describe the experimental setup and training process in simulation. The experiments are conducted using the IsaacLab framework \cite{mittal2023orbit}, with configurations borrowed from IsaacGym \cite{makoviychuk2021isaac}, which support large-scale, parallelized physics simulations on high-performance GPUs.

As shown in Fig. \ref{fig:playground}, the training environment consists of a series of courses (tracks) with progressively increasing difficulty levels. An agent must successfully complete a level to progress to the next; failure to do so results in a fallback to the previous level. The terrain difficulty settings are detailed in Table \ref{tab:terrain_values}, with  a total of 10 discrete difficulty levels for each track.

To evaluate the performance of the trained policies, we utilize two key metrics:
\begin{itemize}
    \item \textbf{Level}: The accomplished environment difficulty level during the training process.
    \item \textbf{Reward}: The cumulative reward obtained during the training process.
\end{itemize}
The average environment difficulty level is used for a fair cross-platform comparison, as each robot may have a different reward structure.

Regarding the training process, we use the Unitree A1 robot as the prototype for initial latent space definition. The L3P framework is first trained on this prototype platform in a multi-track simulated environment. To ensure commonality, all platforms are trained using relative joint position actions.
For training algorithm, we utilize the the implementation of PPO by rsl rl \cite{rudin2022learning}. 

\subsection{Transfer Experiments}
\subsubsection{Generalization from a Single Entity to Diverse Entities}

\textbf{Task Setup:} 
To evaluate the effectiveness of our approach in transferring locomotion policies across different robotic platforms, we conduct transfer experiments on seven widely used legged robots (e.g., A1, Go1, Go2), as illustrated in Fig. \ref{subfig:entities}. 
We compare four baseline methods:
\begin{itemize}
    \item \textbf{MLP}: A fully connected end-to-end policy trained from scratch for each new entity.
    \item \textbf{L3P}: Our proposed method, which learns a latent space representation while optimizing all modules without freezing the latent-to-latent policy.
    \item \textbf{L3P-Transfer}: Our L3P where a  learned latent-to-latent policy on A1 robot backbone remains frozen, and only the observation encoder and action decoder are optimized during adaptation.
    \item \textbf{MLP-Transfer}: An approach where an initial MLP policy is trained on A1 robot and then used as a starting point for continual training.
\end{itemize}

For a fair comparison, all methods use identical learnable parameters and share the same hyper-parameters. Each policy is trained for 1500 iterations.

\textbf{Results:}
Fig. \ref{fig:main_exp_transfer} presents the training curves for the entity transfer experiments. 
Ours L3P method achieves comparable performance to the MLP baseline, validating that our latent space formulation does not hinder optimization, as supported by Theorem \ref{thm:latent_opt_align}. 
Additionally, L3P-Transfer exhibits the fastest convergence across all tested entities, demonstrating superior adaptability. 
In contrast, MLP-Transfer performs well on Unitree robots but suffers significant degradation when applied to other platforms.

A detailed metric comparison is provided in Table \ref{tab.sim_result_large}. Notably, in most tasks, L3P-Transfer maintains the fastest convergence speed, achieving up to 3$\times$ faster convergence. The reward converges more quickly as the terrain difficulty gradually increases during training. However, while the policy stabilizes in terms of reward, it must still progress through each terrain level sequentially, introducing a delay in fully reflecting its convergence capability.

\begin{figure}
\centering
\caption*{(1) Unstable Walking Snapshots}
\begin{subfigure}[b]{0.15\textwidth}
    \includegraphics[width=\textwidth]{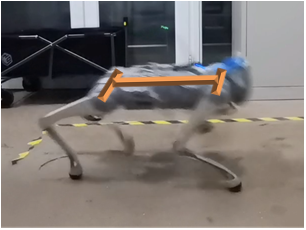}
\end{subfigure}
\hfill
\begin{subfigure}[b]{0.15\textwidth}
    \includegraphics[width=\textwidth]{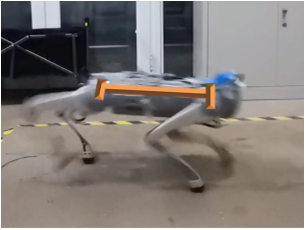}
\end{subfigure}
\hfill
\begin{subfigure}[b]{0.15\textwidth}
    \includegraphics[width=\textwidth]{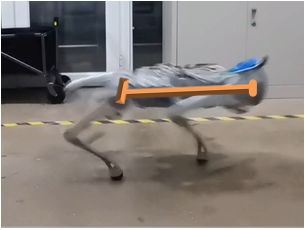}
\end{subfigure} \\
\caption*{(2) Stable Walking Snapshots}
\begin{subfigure}[b]{0.15\textwidth}
    \includegraphics[width=\textwidth]{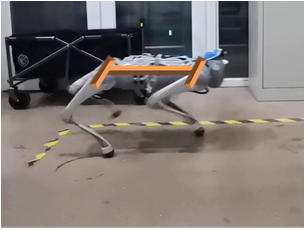}
\end{subfigure}
\hfill
\begin{subfigure}[b]{0.15\textwidth}
    \includegraphics[width=\textwidth]{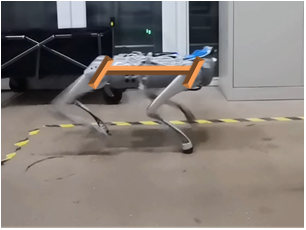}
\end{subfigure}
\hfill
\begin{subfigure}[b]{0.15\textwidth}
    \includegraphics[width=\textwidth]{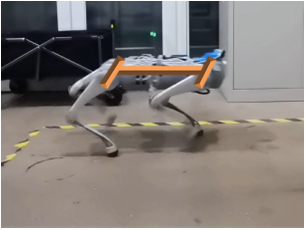}
\end{subfigure} \\
\caption{Comparison of real-world deployment between direct and adapted policies. Three key walking postures are selected for each case. Orange lines indicate the robot's base posture.}
\label{fig:real_exp}
\end{figure}

\subsubsection{Adaptation from Simple to Complex Terrains}
% our method demonstrates a robust capability for zero-shot transfer to new entities on previously unseen terrains. 

\textbf{Task Setup:} A new robot is introduced, with access only to simple flat terrain for fine-tuning, which simulates the scenarios where training a specific robot on complex or target terrains is impractical
The objective is to fine-tune the policy on this simple terrain and subsequently evaluate its performance on complex and challenging terrains.

\textbf{Experimental Details:} 
We compare two methods: L3P-Transfer and MLP-Transfer. 
Both are pre-trained on the Unitree A1 robot navigating complex terrains. 
For evaluation, we utilize the Unitree Go1 and Go2 robots as new entities, fine-tuning them on simple rough terrain.

\textbf{Evaluation Procedure:} 
After training convergence, we assess the policies on new tracks, allowing the robots to traverse as far as possible to gauge their capabilities. 

\textbf{Results:} 
The results are depicted in Fig. \ref{fig:transfer_radar}.
The L3P-Transfer method maintains robust performance across tasks, whereas the MLP-Transfer method suffers from catastrophic forgetting, failing to navigate complex terrains and only demonstrating competence on simple rough terrains.

\subsection{Ablation Study on Latent Variable Dimensions}
\textbf{Task Setup:} To evaluate the impact of latent space dimensionality, we perform ablation studies on both the latent observation and latent action spaces by varying their dimensions and analyzing their effects on performance metrics.

\textbf{Result:}
As illustrated in Fig. \ref{fig:ablation}, we assess the convergence behavior of different configurations. 
For the latent action space, increasing its dimension from 16 to 128 yields no significant performance gain. 
However, reducing the latent action space below a certain threshold degrades performance. 
Notably, even with a minimal latent action dimension of 2, the policy retains basic locomotion capabilities, achieving nonzero terrain levels.

In contrast, the latent observation space consistently operates at a lower dimensionality than the raw observation space. 
Our results indicate that reducing the latent observation dimension from 256 to 64 has little effect on final performance. Even smaller configurations (e.g., 16 and 8) still maintain satisfactory results, demonstrating that latent observations can be efficiently compressed while preserving essential task-relevant information.

\begin{figure}
\centering
\caption*{(1) Simulation Snapshots}
\begin{subfigure}[b]{0.15\textwidth}
    \includegraphics[width=\textwidth]{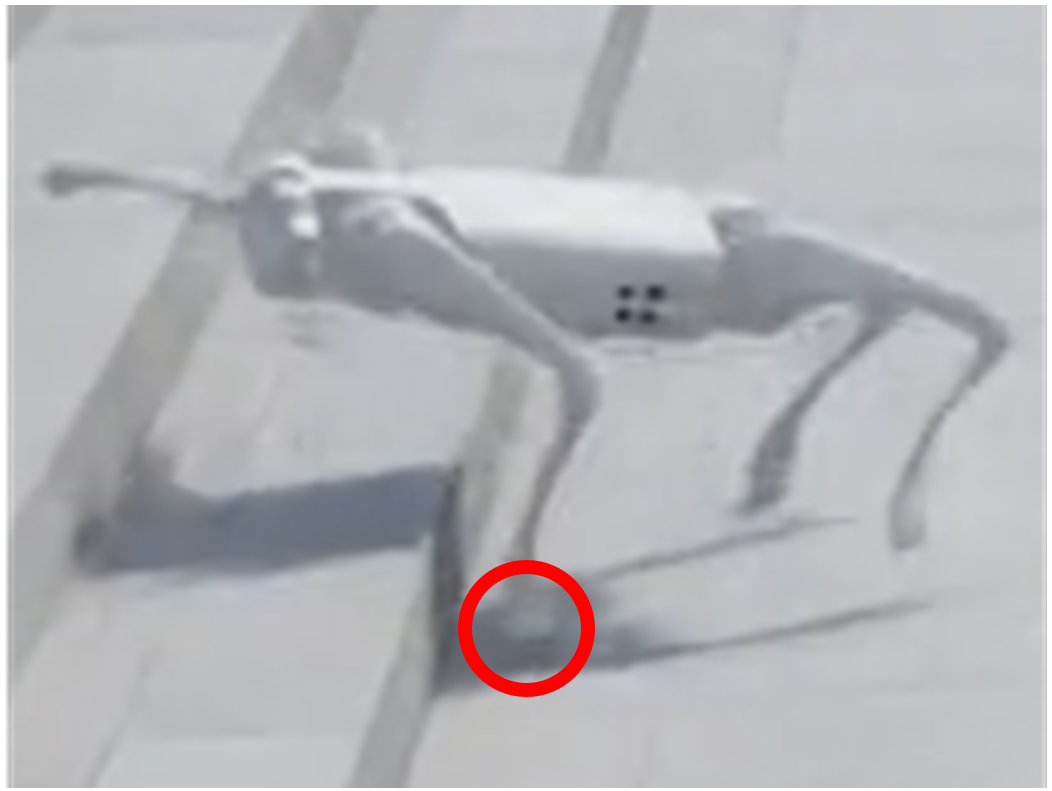}
\end{subfigure}
\hfill
\begin{subfigure}[b]{0.15\textwidth}
    \includegraphics[width=\textwidth]{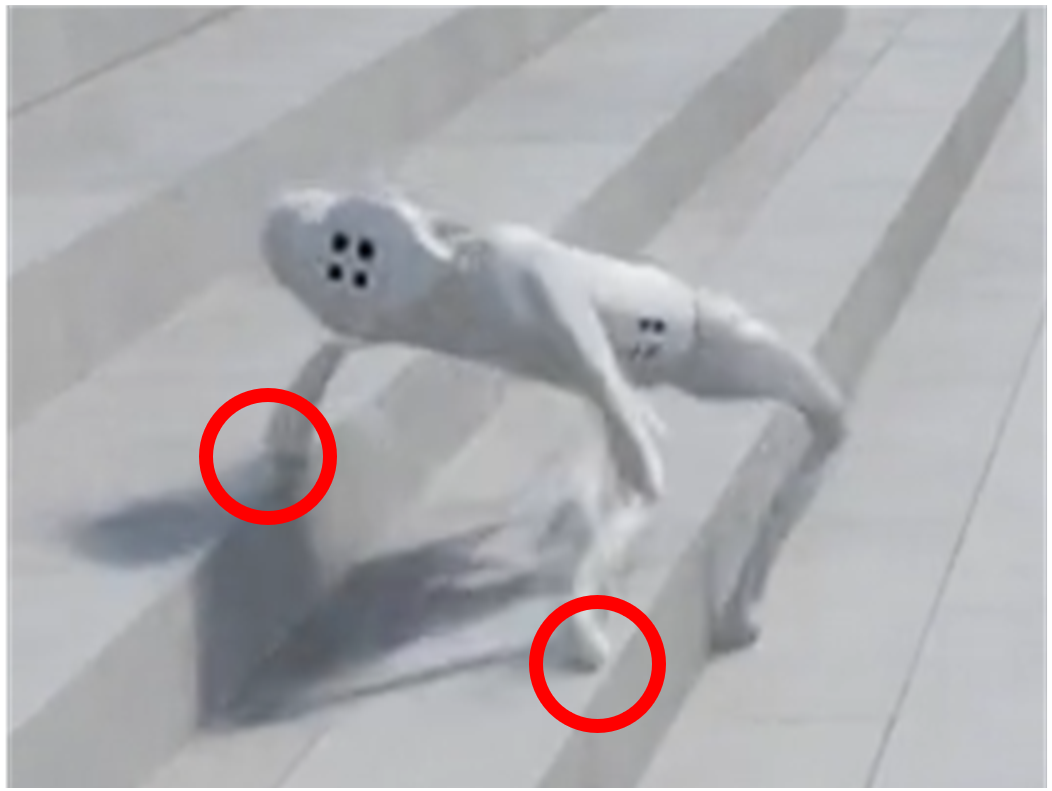}
\end{subfigure}
\hfill
\begin{subfigure}[b]{0.15\textwidth}
    \includegraphics[width=\textwidth]{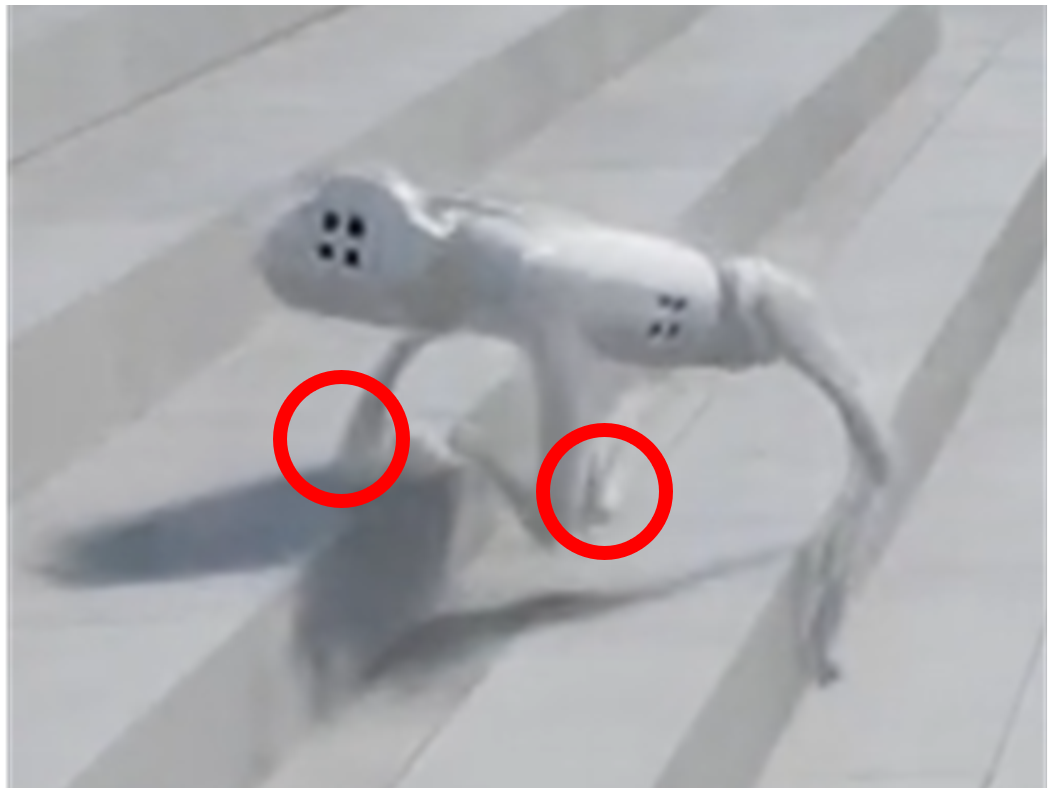}
\end{subfigure} \\
\caption*{(2) Real-World Snapshots}
\begin{subfigure}[b]{0.15\textwidth}
    \includegraphics[width=\textwidth]{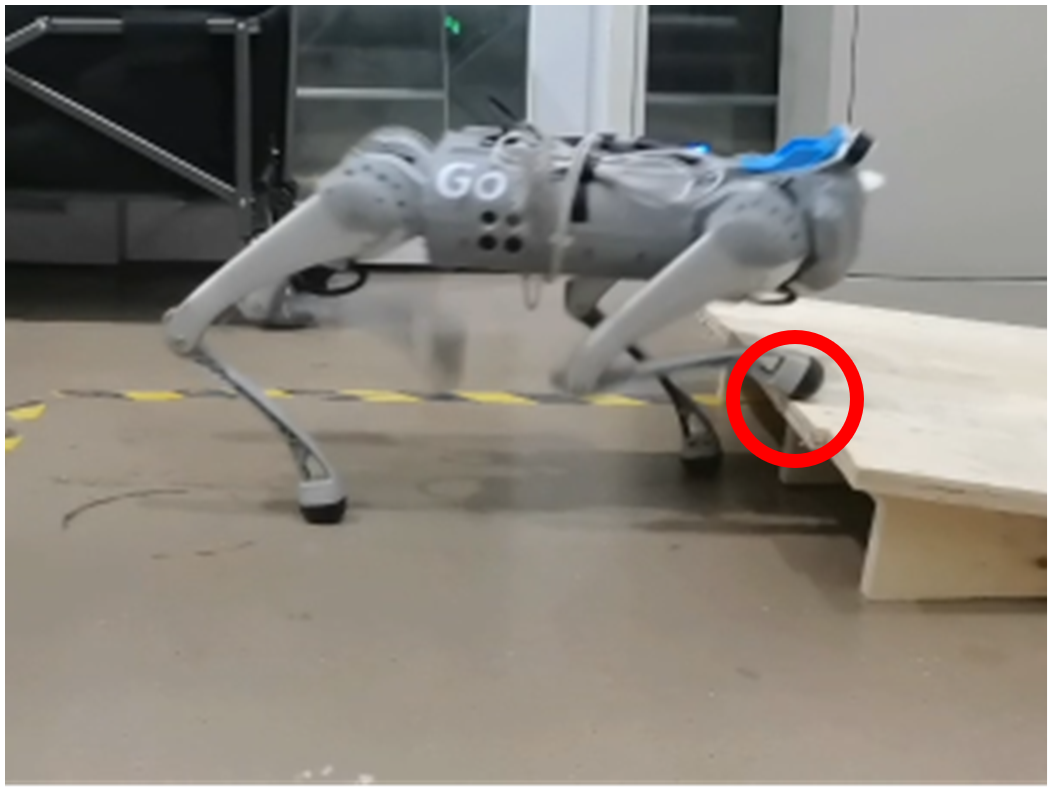}
\end{subfigure}
\hfill
\begin{subfigure}[b]{0.15\textwidth}
    \includegraphics[width=\textwidth]{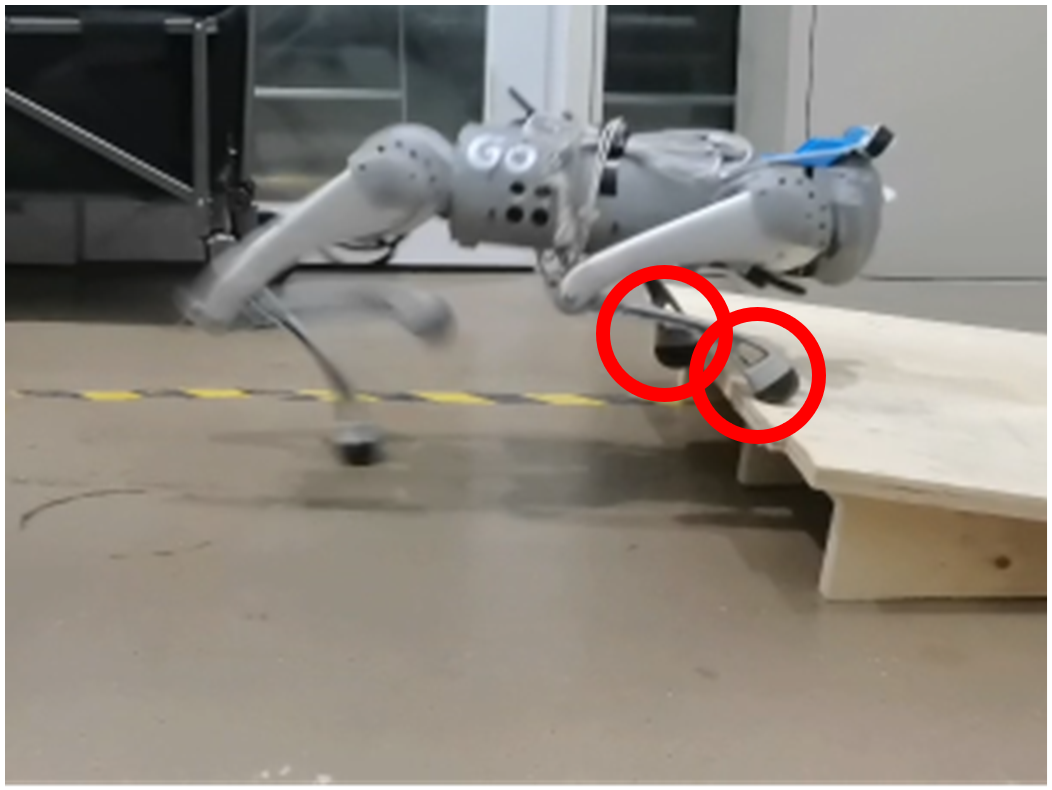}
\end{subfigure}
\hfill
\begin{subfigure}[b]{0.15\textwidth}
    \includegraphics[width=\textwidth]{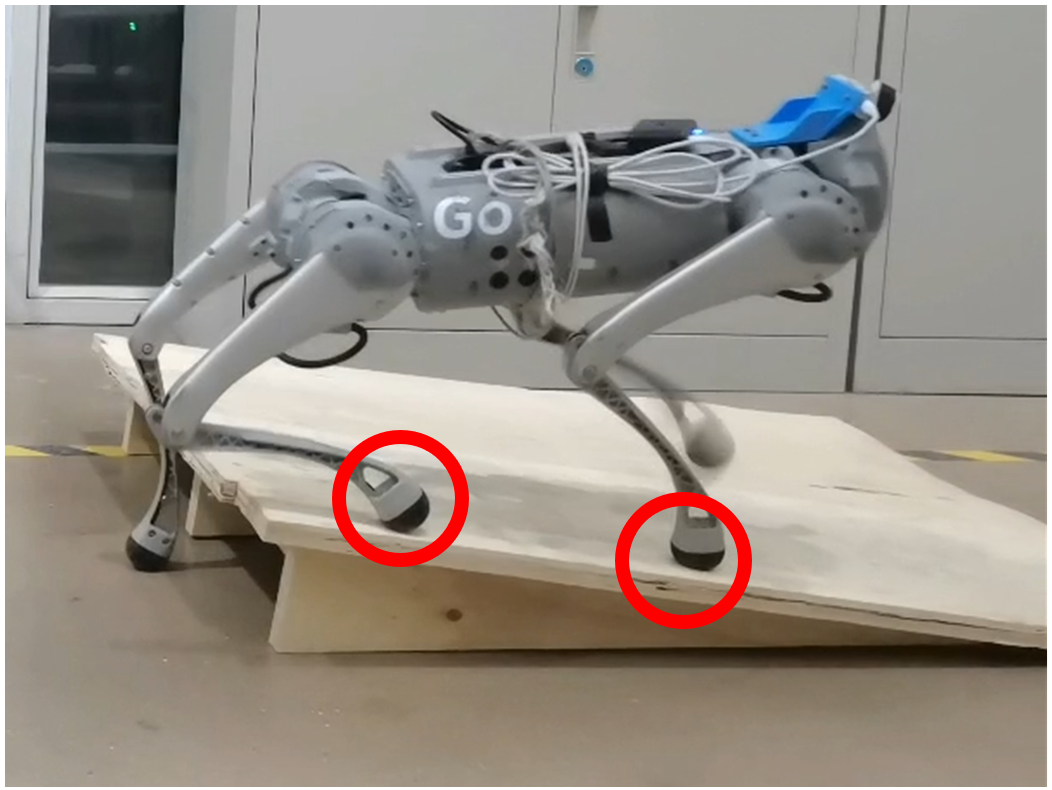}
\end{subfigure} \\
\caption{Comparison of stair traversal between simulation and real-world deployment. Red circles highlight foot contact points. Three key contact frames are selected to illustrate performance.}
\label{fig:real_cmp_sim}
\end{figure}

\subsection{Real-World Experiments}

\textbf{Deployment Pipeline:}  
% To mitigate the sim-to-real gap, we adopt the deployment framework from \cite{zhuangrobot}, utilizing a teacher-student architecture\cite{chen2020learning} to distill a real-world policy.  
For our real-world experiments, we select the Unitree Go1 as the target platform and use the pretrained latent-to-latent policy on A1 robot to start fine-tuning.  
Observations are constructed using only proprioceptive feedback and foot contact forces.  

\textbf{Task Setup:}  
Following a similar methodology as in simulation, we treat the real-world robot as an independent domain, distinct from the simulated training environments.  
We deploy the learned latent policy onto the real robot and fine-tune it in real-world conditions using collected data \cite{smith2022legged} within the L3P framework.  
The fine-tuning process focuses on a walking adjustment task to enhance stability and posture control.  
Finally, we evaluate the robot’s performance on a stair traversal challenge\cite{chamorro2024reinforcement, siekmann2021blind}.  

\textbf{Results:}  
As shown in Fig. \ref{fig:real_exp}, direct deployment without adaptation results in significant instability, leading to poor posture control.  
In contrast, the fine-tuned policy enables the robot to maintain a stable and balanced stance throughout the task.  
For stair traversal, as illustrated in Fig. \ref{fig:real_cmp_sim}, the fine-tuned policy exhibits good behaviors consistent with simulation, particularly in foot contact sequencing.  
The adapted policy successfully enables the robot to navigate the stairs.

%% file: table/terrain_tab.tex
\begin{table}[t]
\centering
\caption{Terrain Configuration Values}
\label{tab:terrain_values}
\begin{threeparttable}
\begin{tabular}{lll}
\toprule
\textbf{Terrain Type} & \textbf{Parameter} & \textbf{Value} \\
\midrule
Pyramid Stairs & step height range & (0.05, 0.8) \\
 & step width & 0.3 \\
 & platform width & 3.0 \\
 & border width & 1.0 \\
 & holes & False \\
\midrule
Inverted Pyramid Stairs & step height range & (0.05, 0.8) \\
 & step width & 0.3 \\
 & platform width & 3.0 \\
 & border width & 1.0 \\
 & holes & False \\
\midrule
Boxes & grid height range & (0.05, 0.8) \\
 & grid width & 0.45 \\
 & platform width & 2.0 \\
\midrule
Random Rough & noise range & (0.02, 0.30) \\
 & noise step & 0.02 \\
 & border width & 0.25 \\
\midrule
Pyramid Slope & slope range & (0.0, 0.8) \\
 & platform width & 2.0 \\
 & border width & 0.25 \\
\midrule
Inverted Pyramid Slope & slope range & (0.0, 0.8) \\
 & platform width & 2.0 \\
 & border width & 0.25 \\
\bottomrule
\end{tabular}
\vspace{3pt}
\begin{tablenotes}\footnotesize
\item[*] The range parameter is responsible for the terrain difficulty, where higher the value is the task is more difficult.
\end{tablenotes}
\end{threeparttable}
\end{table}

%% file: table/main_pics.tex
\begin{figure*}[t]
\centering
\begin{subfigure}[b]{0.24\textwidth}
    \includegraphics[width=\textwidth]{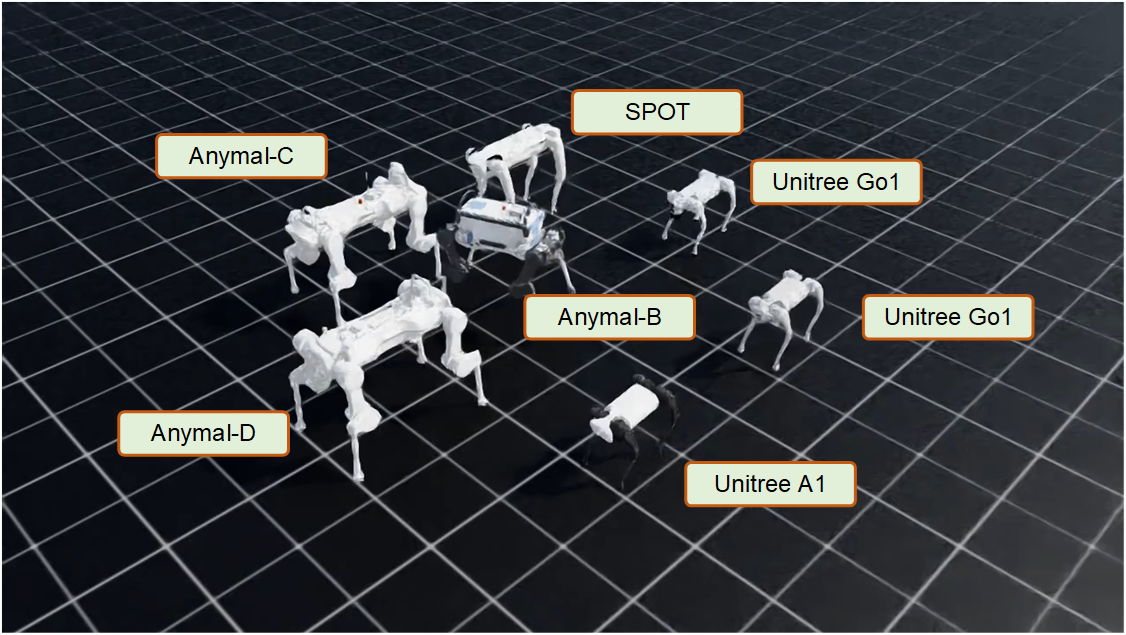}
    \caption{Sketch of entities}
    \label{subfig:entities}
\end{subfigure}
\hfill
\begin{subfigure}[b]{0.24\textwidth}
    \includegraphics[width=\textwidth,trim={0.9cm 0.1cm 1.0cm 1.0cm}, clip]{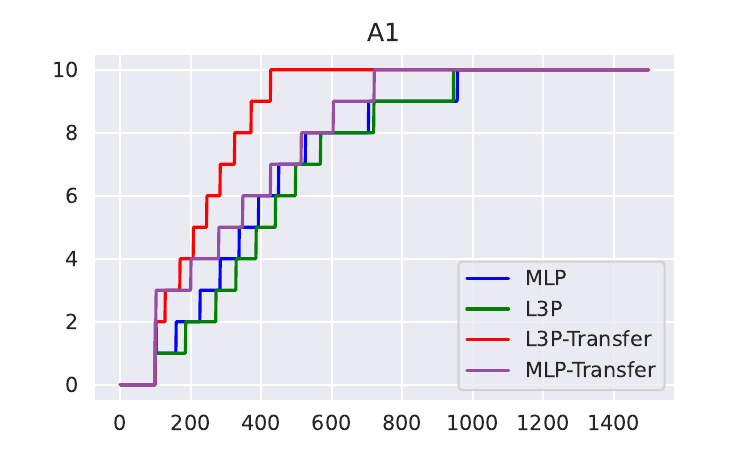}
     \caption{A1}
\end{subfigure}
\hfill
\begin{subfigure}[b]{0.24\textwidth}
    \includegraphics[width=\textwidth,trim={0.9cm 0.1cm 1.0cm 1.0cm}, clip]{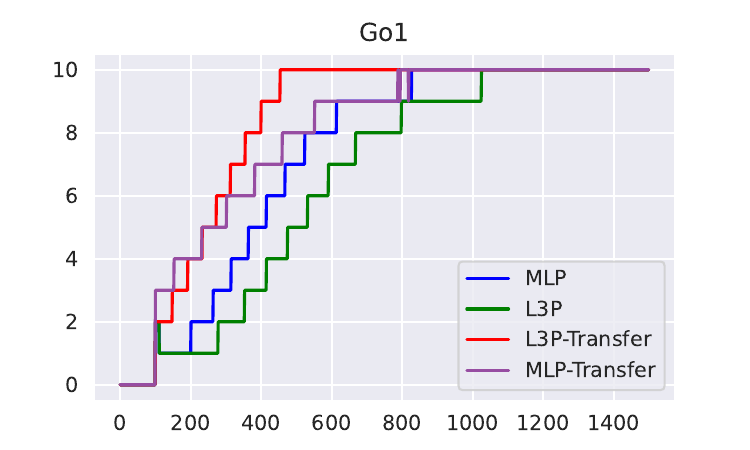}
     \caption{Go1}
\end{subfigure}
\hfill
\begin{subfigure}[b]{0.24\textwidth}
    \includegraphics[width=\textwidth,trim={0.9cm 0.1cm 1.0cm 1.0cm}, clip]{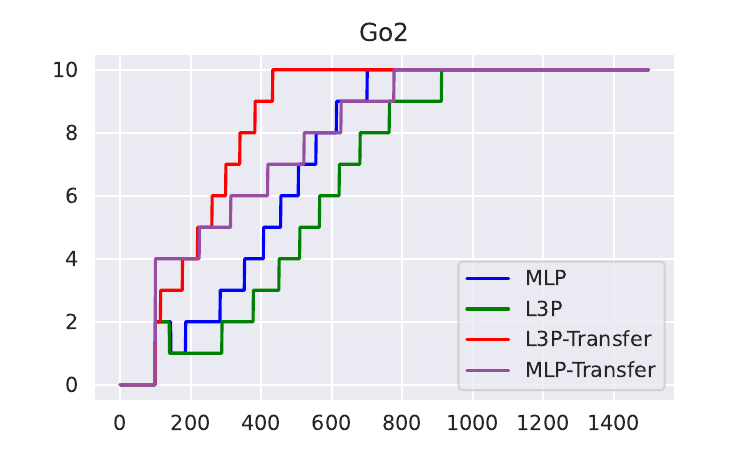}
     \caption{Go2}
\end{subfigure}

\vspace{0cm}

\begin{subfigure}[b]{0.24\textwidth}
    \includegraphics[width=\textwidth,trim={0.9cm 0.1cm 1.0cm 1.0cm}, clip]{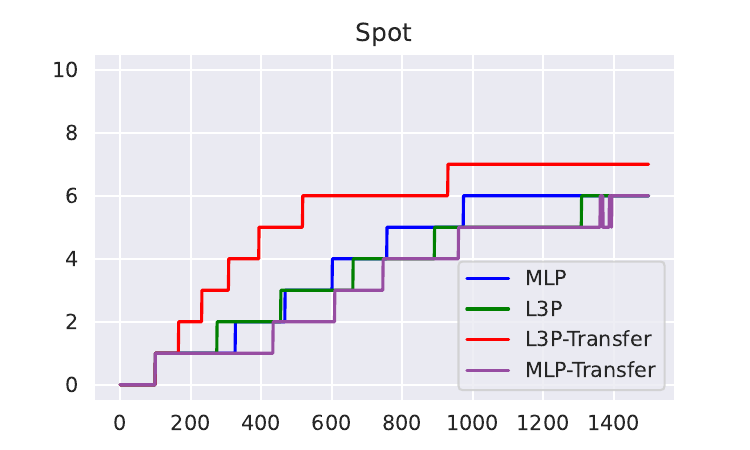}
     \caption{Spot}
\end{subfigure}
\hfill
\begin{subfigure}[b]{0.24\textwidth}
    \includegraphics[width=\textwidth,trim={0.9cm 0.1cm 1.0cm 1.0cm}, clip]{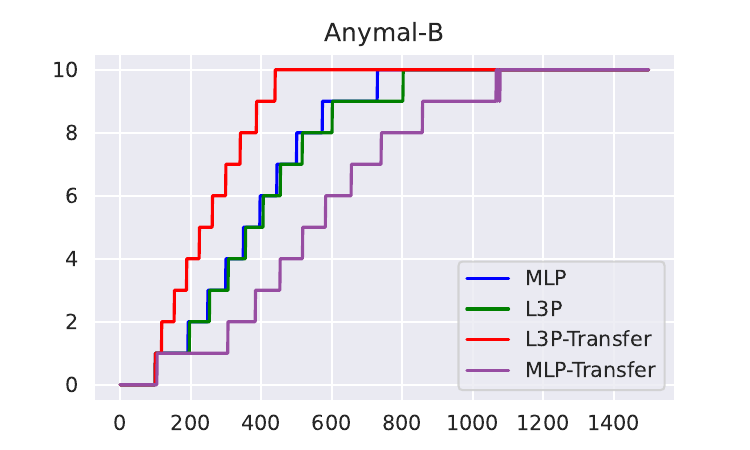}
     \caption{Anymal-B}
\end{subfigure}
\hfill
\begin{subfigure}[b]{0.24\textwidth}
    \includegraphics[width=\textwidth,trim={0.9cm 0.1cm 1.0cm 1.0cm}, clip]{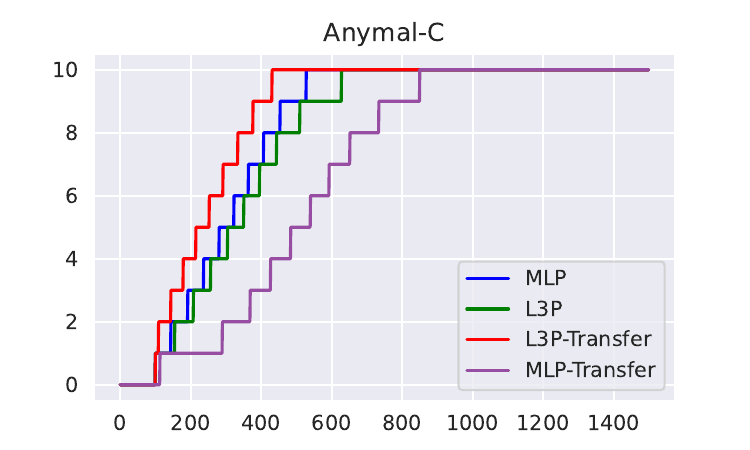}
     \caption{Anymal-C}
\end{subfigure}
\hfill
\begin{subfigure}[b]{0.24\textwidth}
    \includegraphics[width=\textwidth,trim={0.9cm 0.1cm 1.0cm 1.0cm}, clip]{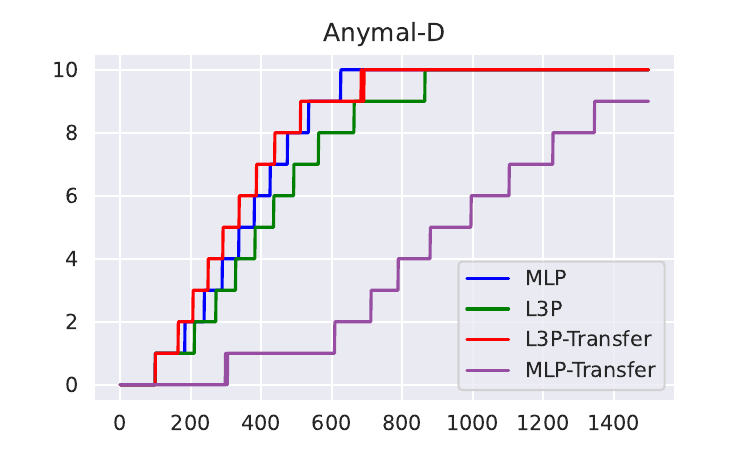}
     \caption{Anymal-D}
\end{subfigure}
\caption{
Performance verification via multi-type legged robots. We conduct experiment on seven classical legged robot (A1, Go1, Go2, Anymal-B, Anymal-C, Anymal-D \cite{hutter2016anymal}, Spot) and record the accomplished difficulty level during the training. Each curve is the average result between 5 seeds.}
\label{fig:main_exp_transfer}
\end{figure*}

%% file: table/mass_cmp.tex
\begin{table*}[t]
\centering
\caption{Accuracy and Performance Comparison}
\label{tab.sim_result_large}
\begin{threeparttable}
\begin{tabular}{lc|ccc|ccc|cc}
\toprule
\midrule
& & \multicolumn{3}{c}{Level Metric} & \multicolumn{3}{c}{Reward Metric} & \multicolumn{2}{c}{Tracking Accuracy}\\
Task & Method & 85\% Cv.$\downarrow$ & 95\% Cv.$\downarrow$ & Level $\uparrow$ & 85\% Cv.$\downarrow$ & 95\% Cv.$\downarrow$ & Reward $\uparrow$ & Vel. Err.$\downarrow$ & Yaw Err.$\downarrow$ \\
\midrule

\multirow{4}{*}{A1} 
& MLP          & 796.00  & 1156.00 & 5.88  $\pm$ 0.28 & 851.00  & 1328.00 & 22.04  $\pm$ 2.04 & 1.77 \% $\pm$ 0.12 & 2.01 \% $\pm$ 0.06 \\ 
& L3P          & 760.00  & 978.00  & 5.77  $\pm$ 0.07 & 740.00  & 1202.00 & 23.16  $\pm$ 0.95 & 1.74 \% $\pm$ 0.06 & 2.01 \% $\pm$ 0.09 \\ 
& MLP-Trans. & 648.00  & 933.00  & \textbf{5.98}  $\pm$ 0.23 & 545.00  & 1337.00 & \textbf{24.56}  $\pm$ 0.82 & \textbf{1.16} \% $\pm$ 0.10 & \textbf{0.97} \% $\pm$ 0.06 \\ 
& L3P-Trans. & \textbf{394.00}  & \textbf{462.00}  & 5.95  $\pm$ 0.03 & \textbf{179.00}  & \textbf{226.00}  & 24.36  $\pm$ 0.63 & 1.78 \% $\pm$ 0.04 & 1.97 \% $\pm$ 0.03 \\ 
\midrule
\multirow{4}{*}{Go1}
& MLP          & 691.00  & 975.00  & 5.95  $\pm$ 0.10 & 463.00  & 1238.00 & 23.42  $\pm$ 0.83 & 1.75 \% $\pm$ 0.06 & 1.70 \% $\pm$ 0.11 \\ 
& L3P          & 830.00  & 1078.00 & 5.78  $\pm$ 0.06 & 664.00  & 1190.00 & 23.14  $\pm$ 1.71 & 1.73 \% $\pm$ 0.05 & 1.76 \% $\pm$ 0.04 \\ 
& MLP-Trans. & 591.00  & 930.00  & 5.85  $\pm$ 0.28 & 296.00  & 1086.00 & 23.45  $\pm$ 2.68 & 1.26 \% $\pm$ 0.05 & \textbf{0.88} \% $\pm$ 0.02 \\ 
& L3P-Trans. & \textbf{428.00}  & \textbf{512.00}  & \textbf{6.02}  $\pm$ 0.06 & \textbf{181.00}  & \textbf{234.00}  & \textbf{24.45}  $\pm$ 0.87 & \textbf{1.17} \% $\pm$ 0.04 & 1.72 \% $\pm$ 0.08 \\ 
\midrule
\multirow{4}{*}{Go2}
& MLP          & 663.00  & 1080.00 & 6.10  $\pm$ 0.05 & 615.00  & 1106.00 & 26.19  $\pm$ 0.36 & 1.77 \% $\pm$ 0.06 & 1.93 \% $\pm$ 0.10 \\ 
& L3P          & 780.00  & 948.00  & 5.76  $\pm$ 0.09 & 1015.00 & 1291.00 & 25.30  $\pm$ 0.51 & 1.69 \% $\pm$ 0.06 & 1.92 \% $\pm$ 0.12 \\ 
& MLP-Trans. & 678.00  & 890.00  & 5.96  $\pm$ 0.10 & 672.00  & 1140.00 & 25.67  $\pm$ 1.43 & 1.38 \% $\pm$ 0.24 & \textbf{1.48} \% $\pm$ 0.61 \\ 
& L3P-Trans. & \textbf{412.00}  & \textbf{477.00}  & \textbf{6.08}  $\pm$ 0.04 & \textbf{214.00}  & \textbf{279.00}  & \textbf{25.88}  $\pm$ 0.43 & \textbf{1.25} \% $\pm$ 0.06 & 1.96 \% $\pm$ 0.05 \\ 
\midrule
\multirow{4}{*}{Any.B}
& MLP          & 609.00  & 825.00  & \textbf{5.86}  $\pm$ 0.04 & 292.00  & 378.00  & 9.57   $\pm$ 0.63 & 2.00 \% $\pm$ 0.10 & 1.93 \% $\pm$ 0.08 \\ 
& L3P          & 642.00  & 917.00  & 5.85  $\pm$ 0.02 & 313.00  & 402.00  & 9.85   $\pm$ 0.26 & \textbf{1.92} \% $\pm$ 0.01 & 1.96 \% $\pm$ 0.14 \\ 
& MLP-Trans. & 897.00  & 1136.00 & 5.84  $\pm$ 0.00 & 514.00  & 848.00  & 9.75   $\pm$ 0.00 & 1.93 \% $\pm$ 0.00 & \textbf{1.67} \% $\pm$ 0.00 \\ 
& L3P-Trans. & \textbf{399.00}  & \textbf{459.00}  & 5.83  $\pm$ 0.02 & \textbf{147.00}  & \textbf{174.00}  & \textbf{13.20}  $\pm$ 0.51 & 1.96 \% $\pm$ 0.10 & 1.88 \% $\pm$ 0.04 \\ 
\midrule
\multirow{4}{*}{Any.C}
& MLP          & 499.00  & 624.00  & 6.10  $\pm$ 0.02 & 221.00  & 299.00  & 12.02  $\pm$ 0.43 & 1.93 \% $\pm$ 0.05 & 1.78 \% $\pm$ 0.02 \\ 
& L3P          & 569.00  & 829.00  & 6.07  $\pm$ 0.04 & 266.00  & 357.00  & 11.38  $\pm$ 0.21 & 2.00 \% $\pm$ 0.07 & \textbf{1.73} \% $\pm$ 0.06 \\ 
& MLP-Trans. & 817.00  & 1035.00 & \textbf{6.17}  $\pm$ 0.00 & 457.00  & 593.00  & 10.55  $\pm$ 0.00 & 2.71 \% $\pm$ 0.00 & 2.04 \% $\pm$ 0.00 \\ 
& L3P-Trans. & \textbf{401.00}  & \textbf{488.00}  & 5.99  $\pm$ 0.06 & \textbf{123.00}  & \textbf{154.00}  & \textbf{13.41}  $\pm$ 0.97 & \textbf{1.92} \% $\pm$ 0.05 & 1.79 \% $\pm$ 0.12 \\ 
\midrule
\multirow{4}{*}{Any.D}
& MLP          & 588.00  & 803.00  & \textbf{6.12}  $\pm$ 0.03 & 307.00  & 408.00  & 11.28  $\pm$ 0.67 & 1.88 \% $\pm$ 0.08 & 1.70 \% $\pm$ 0.03 \\ 
& L3P          & 763.00  & 1039.00 & 6.03  $\pm$ 0.05 & 403.00  & 945.00  & 10.42  $\pm$ 0.43 & 1.86 \% $\pm$ 0.05 & 1.74 \% $\pm$ 0.04 \\ 
& MLP-Trans. & 1285.00 & 1414.00 & 5.28  $\pm$ 0.00 & 1233.00 & 1431.00 & 9.80   $\pm$ 0.00 & 1.95 \% $\pm$ 0.00 & 1.79 \% $\pm$ 0.00 \\ 
& L3P-Trans. & \textbf{519.00}  & \textbf{695.00}  & 5.69  $\pm$ 0.10 & \textbf{172.00}  & \textbf{229.00}  & \textbf{11.49}  $\pm$ 1.06 & \textbf{1.84} \% $\pm$ 0.01 & \textbf{1.70} \% $\pm$ 0.07 \\ 
\midrule
\multirow{4}{*}{Spot}
& MLP          & 834.00    & 1283.00 & 3.54  $\pm$ 0.03 & 261.00  & 389.00  & 162.26 $\pm$ 13.0 & 3.83 \% $\pm$ 0.07 & 2.10 \% $\pm$ 0.03 \\ 
& L3P          & 877.00    & 936.00    & 3.24  $\pm$ 0.02 & 258.00  & 434.00  & 160.40 $\pm$ 5.47 & 3.82 \% $\pm$ 0.09 & 2.12 \% $\pm$ 0.06 \\ 
& MLP-Trans. & 897.00    & 1344.00    & 3.20  $\pm$ 0.00 & 429.00  & 600.00  & 135.26 $\pm$ 0.00 & \textbf{2.17} \% $\pm$ 0.00 & \textbf{1.34} \% $\pm$ 0.00 \\ 
& L3P-Trans. & \textbf{412.00}    & \textbf{864.00}  & \textbf{3.74}  $\pm$ 0.02 & \textbf{171.00}  & \textbf{217.00}  & \textbf{181.99} $\pm$ 3.79 & 3.75 \% $\pm$ 0.04 & 2.12 \% $\pm$ 0.07 \\ 

\midrule
\bottomrule
\end{tabular}

\vspace{3pt}
\begin{tablenotes}\footnotesize
\item[*] All experiments are conducted with 5 random seeds. The reported values in the table represent the average results along with their variance.  
\item[*] \textbf{Cv.} stands for convergence step; \textbf{Trans.} stands for Transfer.
\item[*] The convergence step refers to the number of update steps required to reach a certain percentage of the final rewards or curriculum threshold, reflecting the convergence capability.  
\item[*] \textbf{Err.} denotes tracking error, while \textbf{Vel.} represents x-y velocity. 
\item[*] Tracking error for x-y velocity and yaw is measured as the deviation between the given command and the actual state, indicating the policy's tracking accuracy. However, tracking difficulty varies across different levels.  
\item[*] \textbf{Any.} refers to Anymal robots.  
\item[*] $\uparrow$ indicates that a higher value is better; $\downarrow$ indicates that a lower value is better.  
\item[*] Bold values indicate the best performance.  
\end{tablenotes}
\end{threeparttable}
\end{table*}